%% file: knn_icml2012.tex
\documentclass{article}

\usepackage{graphicx,amssymb} 
\usepackage{subfigure} 
\usepackage{natbib}

\usepackage{algorithm}

\usepackage[accepted]{icml2012}

\input{ules_latex_defaults.tex}

\renewcommand{\ulesqed}{\mbox{}\hfill$\Box$}
\newcommand{\SP}{\mathcal{SP}}

\begin{document} 

\twocolumn[
\icmltitle{Shortest path distance in random $k$-nearest neighbor graphs}
\icmlauthor{Morteza Alamgir$^{1}$}{morteza@tuebingen.mpg.de}
\icmlauthor{Ulrike von Luxburg$^{1,2}$}{ulrike.luxburg@tuebingen.mpg.de}
\icmladdress{ $^{1}$ Max Planck Institute for Intelligent Systems,
 T\"{u}bingen, Germany\\
$^{2}$ Department of Computer Science, University of Hamburg, Germany}

\icmlkeywords{shortest path, nearest neighbor graph, density based distance, limit distance}
\vskip 0.2in
]

\begin{abstract} Consider a weighted or unweighted $k$-nearest
  neighbor graph that has been built on $n$ data points drawn randomly
  according to some density $p$ on $\R^d$. We study the convergence of
  the shortest path distance in such graphs as the sample size 
  tends to infinity.  We prove that for unweighted $\knn$ graphs, this
  distance converges to an unpleasant distance function on the
  underlying space whose properties are detrimental to machine
  learning. We also study the behavior of the shortest path distance in weighted
  $\knn$ graphs.  
\end{abstract} 

\section{Introduction}
\label{sec:int}
The shortest path distance is the most fundamental distance function between
vertices in a graph, and it is widely used in computer science and
machine learning. In this paper we want to understand the geometry
induced by the shortest path distance in
randomly generated geometric graphs like $k$-nearest neighbor graphs.

Consider a neighborhood graph $G$ built from an i.i.d. sample $X_1, ...,
X_n$ drawn according to some density $p$ on $ \Xcal \subset \R^d$ (for exact
definitions see Section \ref{sec:def}). Assume that the sample size $n$ goes to infinity. Two questions arise about the behavior of the shortest path distance between fixed points in this graph:

1. \textbf{Weight assignment:} Given a distance measure $D$
on $\Xcal$, how can we assign edge weights such that the 
shortest path distance in the graph converges to $D$? 

2. \textbf{Limit distance:} Given a function $h$ that assigns weights
of the form $h(\|X_i - X_j\|)$ to edges in $G$, what is the
limit of the shortest path distance in this weighted graph as $n \to \infty$?


The first question has already been studied in some special
cases. \citet{TenSilLan00withproofs} discuss the case of $\eps$- and
$\knn$ graphs when $p$ is \textit{uniform} and $D$ is the geodesic
distance. \citet{Sajama2005} extend these results to $\eps$-graphs
from a general density $p$ by introducing edge weights that depend on an
explicit estimate of the underlying density. In a recent preprint, \citet{Hwang2012}
consider completely connected graphs whose vertices come from a
general density $p$ and whose edge weights  are powers
of distances. 


There is little work regarding
the second question. \citet{TenSilLan00withproofs} answer the question
for a very special case with $h(x)=x$ and uniform
$p$. \citet{Hwang2012} study the case $h(x)=x^a$, $a>1$ for arbitrary
density $p$.

We have a more general point of view.  In Section
\ref{sec:sp_weighted} we show that depending on properties of the
function $h(x)$, the shortest path distance operates in different
regimes, and we find the limit of the shortest path distance for
particular function classes of $h(x)$. Our method also reveals a
direct way to answer the first question without explicit density
estimation.

An interesting special case is the unweighted $\knn$ graph, which
corresponds to the constant weight function
$h(x)=1$. 
We show that the shortest path distance on unweighted $\knn$-graphs
converges to a limit distance on $\Xcal$ that does {\em not} conform
to the natural intuition and induces a geometry on $\Xcal$ that can be
detrimental for machine learning applications.


Our results have implications for many machine learning algorithms,
see Section \ref{sec:conseq} for more discussion. (1) The shortest paths based
on unweighted $\knn$ graphs prefer to go through low density regions, and
they even accept large detours if this avoids passing through high
density regions (see Figure \ref{fig:path1} for an
illustration). This is exactly the opposite of what we would
like to achieve in most applications. 
%
(2) For manifold learning algorithms like Isomap, unweighted $\knn$
graphs introduce a
fundamental bias that leads to huge distortions in
the estimated manifold structure (see Figure \ref{fig:isomap_knn} for
an illustration). 
%
(3) In the area of semi-supervised learning, a standard approach is to
construct a graph on the  sample points, then compute
a distance between vertices of the graph, and finally use a standard
distance-based classifier 
to label the unlabeled points (e.g., \citealp{Sajama2005} and \citealp{Bijral11}). 
The crucial property exploited in this
approach is that distances between points should be small if they are
in the same high-density region. Shortest path distances in unweighted
$\knn$ graphs and their limit distances do exactly the opposite, so
they can be misleading for this approach.


\begin{figure}[t]
\centering
\includegraphics[trim = 30mm 20mm 40mm 13mm, clip,
width=0.7\columnwidth,height=0.15\textheight]{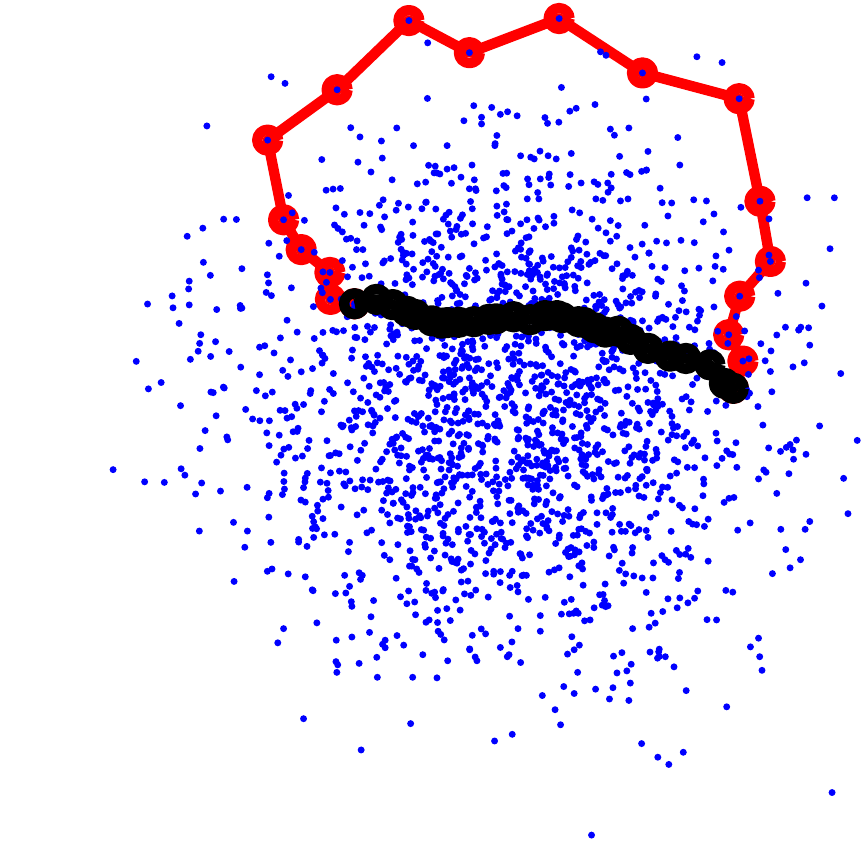}
\vspace{-10pt}
\caption{The shortest path based on an  unweighted (red) and Euclidean
  weighted (black) $\knn$ graph.}
\label{fig:path1}
\end{figure}

\begin{figure}
	\begin{minipage}{.49\linewidth}
		\includegraphics[width=\columnwidth]{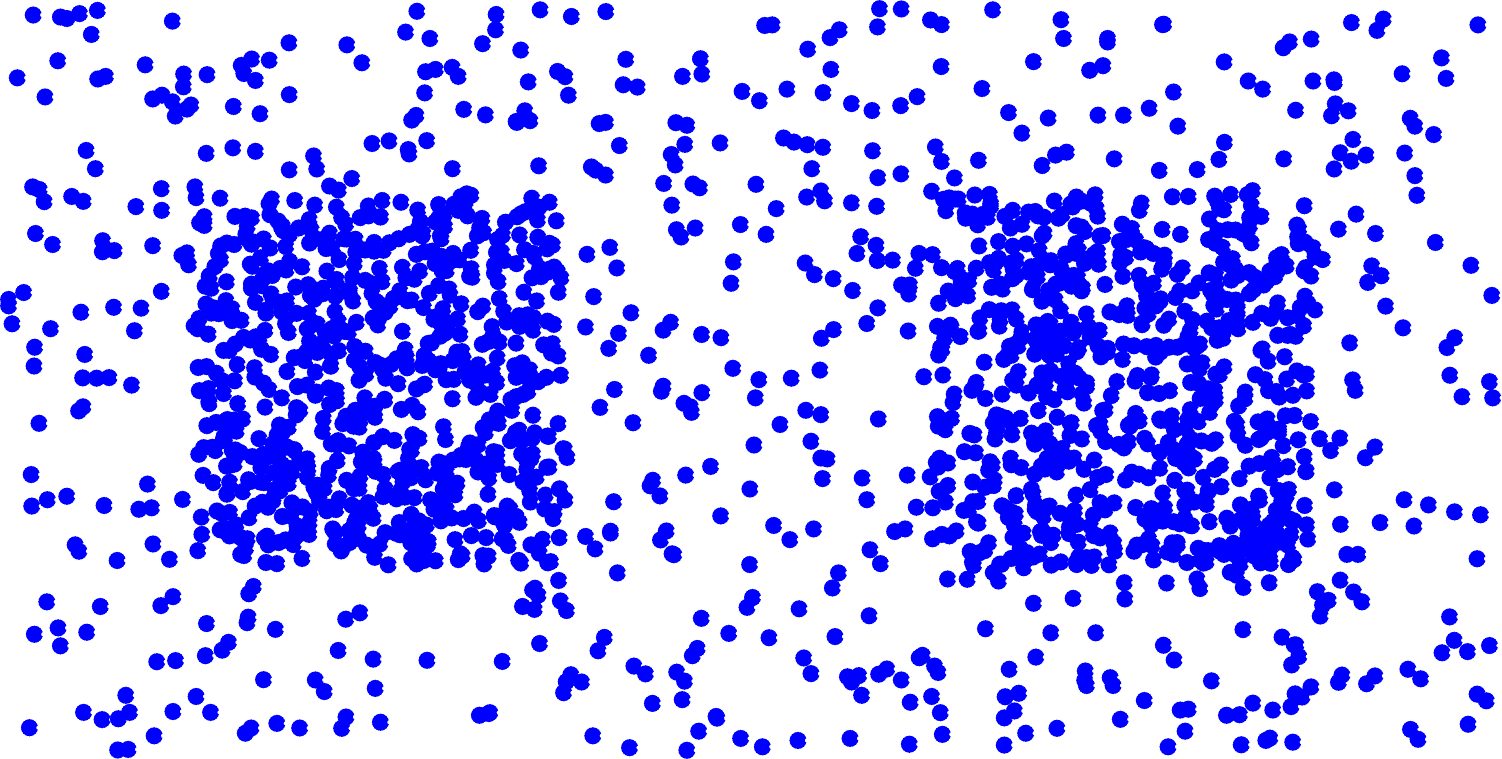}
	\end{minipage}
	\begin{minipage}{.49\linewidth}
		\includegraphics[width=\columnwidth]{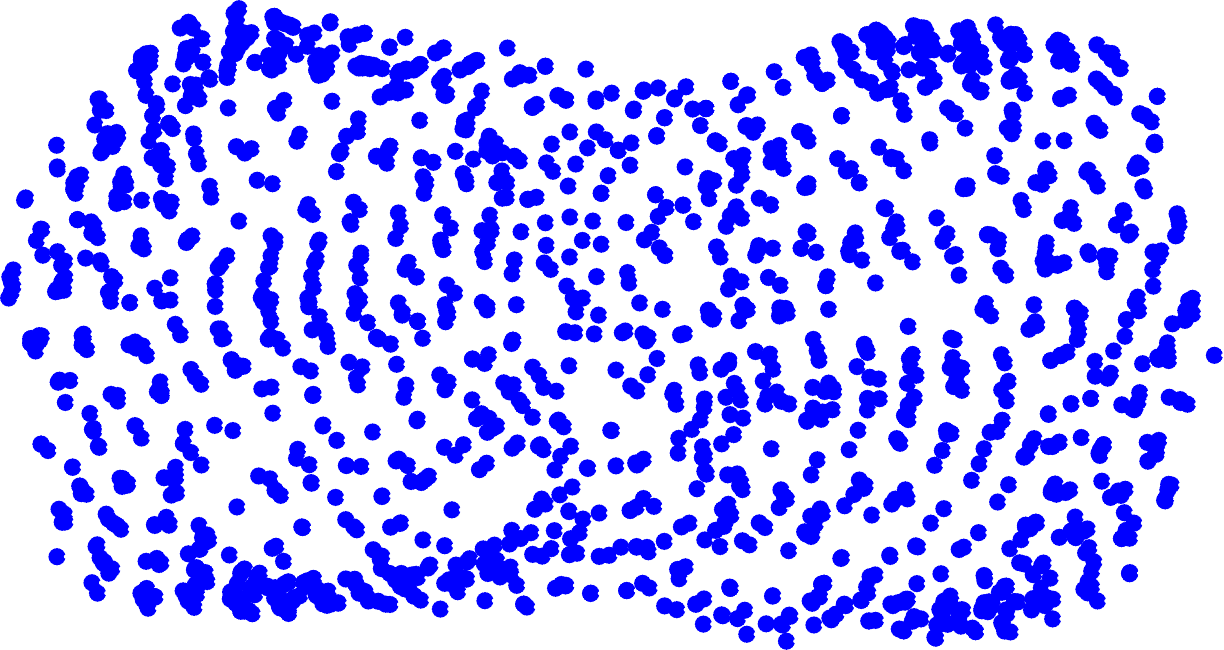}
	\end{minipage}
    \vspace{-10pt}
	\caption{Original data (left) and its Isomap
          reconstruction based on an unweighted $\knn$ graph (right).}
	\vspace{-10pt}
	\label{fig:isomap_knn}
\end{figure}


\section{Basic definitions}
\label{sec:def}


Consider a closed, connected subset $\Xcal \subseteq \mathbb{R}^d$ that
is endowed with a density function $p$ with respect to the Lebesgue
measure. For the ease of presentation we assume for the rest of the
paper that the density $p$ is Lipschitz continuous with Lipschitz
constant $L$ and bounded away from $0$ by $p_{\min} > 0$. 
%
%
To simplify
notation later on, we define the shorthand $q(x) := (p(x))^{1/d}$. 

We will
consider different metrics on $\Xcal$.  A ball with respect to a
particular metric $D$ in $\Xcal$ will be written as $ B(x,r,D) :=
\{y \in \Xcal \mid D(x,y) \leq r \}$. We denote the Euclidean volume
of the unit ball in $\mathbb{R}^d$ by $\eta_d$.

Assume the finite dataset $X_1,...,X_n$ has been drawn i.i.d according
to $p$. We build a geometric graph $G = (V,E)$ that has the data
points as vertices and connects vertices that are close. Specifically,
for the $\knn$ graph we connect $X_i$ with $X_j$ if
$X_i$ is among the $k$ nearest neighbors of $X_j$ or vice versa. For
the $\epsilon$-graph, we connect $X_i$ and $X_j$ whenever their Euclidean
distance satisfies $\|X_i - X_j\| \leq \epsilon$. In this paper, all graphs
are undirected, but might carry edge weights $w_{ij} \geq 0$.  In
unweighted graphs, we define the length of a path by its number of
edges, in weighted graphs we define the length of a path by the sum of
the edge weights along the path. In both cases, the shortest path ($\SP$)
distance $D_{sp}(x,y)$ between two vertices $x,y \in V$ is  the
length of the shortest path connecting them. 

Let $f$ be a positive continuous scalar function defined on
$\Xcal$. For a given path $\gamma$ in $\Xcal$ that connects
$x$ with $y$ and is parameterized by $t$, we define the
$f$-length of the path as
\[
D_{f,\gamma} = \int_{\gamma}{f(\gamma(t))|\gamma'(t)| dt}.
\]
This expression is also known as the \textit{line integral} along $\gamma$ with
respect to $f$. The $f$-geodesic path between $x$ and $y$ is the path
with minimum $f$-length.

The $f$-length of the geodesic path is called the $f$-distance
between $x$ and $y$. We denote it by $D_{f}(x,y)$. If $f(x)$ is a
function of the density $p$ at $x$, then the $f$-distance is sometimes called
a {density based distance} \cite{Sajama2005}. 


The $f$-distance on $\Xcal$ is a metric, and
in particular it satisfies the
triangle inequality. Another useful property is that for a
point $u$ on the $f$-geodesic path between $x$ and $y$ we have
$D_f(x,y)=D_f(x,u)+D_f(u,y)$.

The function $f$ determines the behavior of the $f$-distance. When
$f(x)$ is a monotonically decreasing function of density $p(x)$,
passing through a high density region will cost less than passing
through a low density region. It works the other way round when $f$ is a monotonically increasing function of density. 
A constant function does not impose any preference between low and high density regions.

The main purpose of this paper is to study the relationship between the $\SP$ distance in various geometric graphs and particular
$f$-distances on $\Xcal$. For example, in Section
\ref{sec:sp_unweighted} we show that the $\SP$ distance in unweighted
$\knn$ graphs converges to the $f$-distance with $f(x)
= p(x)^{1/d}$.

In the rest of the paper, all statements refer to points
$x$ and $y$ in the interior of $\Xcal$ such that their $f$-geodesic
path is bounded away from the boundary of $\Xcal$. 

\section{Shortest paths in unweighted graphs}
\label{sec:sp_unweighted}

In this section we study the behavior of the shortest path distance in the family of \textit{unweighted} $\knn$ graphs. 
We show that the rescaled graph $\SP$ distance converges to the $q$-distance in the original space $\Xcal$.  

\begin{theorem}[$\SP$ limit in unweighted $\knn$ graphs]
\label{th:main} Consider the unweighted $\knn$ graph $G_n$ based on the
i.i.d. sample $X_1, ..., X_n \in \Xcal$ from the density $p$. Choose $\lambda$ and $a$ such that 
\[
\lambda \geq \frac{4 L}{\eta_d^{1/d} p_{\min}^{1+1/d} } \Big(\frac{k}{n}\Big)^{1/d} 
\; , \;
a < 1 - \log_{k}\Big( 4^d(1+\lambda)^2 \Big).
\]
Fix two points $x=X_i$ and $y=X_j$. 
Then there exist $e_1(\lambda, k)$,$e_2(\lambda, k, n)$,$e_3(\lambda)$ (see below for explicit definitions) such that with probability at least $1 - 3 e_3 n\exp(-\lambda^2 k^a/6) $
we have 
\[
e_1 D_q(x,y)
\leq 
e_2 D_{sp}(x,y) 
\leq
 D_q(x,y) - e_2.
\]
Moreover if $n\to \infty$, $k\to \infty$, $k/n \rightarrow 0$, $\lambda \rightarrow 0 $ and
$\lambda^2 k^a / \log(n)  \rightarrow \infty$, then the probability converges to $1$ and $(k / (\eta_d n ))^{1/d} D_{sp}(x,y) $ converges to $D_q(x,y)$ in probability.
\end{theorem}

The convergence conditions on $n$ and $k$ are the ones to be expected for random
geometric graphs. The condition $\lambda^2 k^a / \log(n)  \rightarrow
\infty$ is slightly stronger than the usual $k / \log(n)  \rightarrow
\infty$ condition. This condition is satisfied as soon as $k$ is of
the order a bit larger than $\log(n)$. For example $k \approx \log(n)^{1+\alpha}$ with a small $\alpha$ will work. For $k$ smaller
than $\log(n)$, the graphs are not connected anyway (see e.g. \citealp{Penrose99}) and are unsuitable for machine learning applications.

Before proving  Theorem \ref{th:main}, we need to state a couple of
propositions and lemmas. We start by introducing some ad-hoc notation:

\begin{definition} {\bf (Connectivity parameters)}
\label{def:connectivity} Consider a geometric graph based on a fixed set of points $X_1, ...,
X_n \in \R^d$. Let $r_{low}$
be a real number such that $D_f(X_i,X_j) \leq r_{low}$
implies that $X_i$ is connected to $X_j$ in the graph. Analogously, consider $r_{up}$ to be a real number such that $D_f(X_i,X_j) \geq r_{up}$ implies that  $X_i$ is not connected to $X_j$ in the graph. 
\end{definition}
\begin{definition}{\bf (Dense sampling assumption)}
Consider a graph $G$ with connectivity parameters $r_{low}$ and
$r_{up}$. We say that it satisfies the dense sampling assumption if there exists an $\varsigma<r_{low}/4$ such that for all $x \in \Xcal$ there exists a vertex $y$ in the graph with $D_f(x,y) \leq \varsigma$.
\end{definition}
\begin{proposition}[Bounding $D_{sp}$ by $D_f$]
\label{pr:unweighted_sp} Consider any unweighted geometric graph based on
a fixed set 
$X_1, ..., X_n \in \Xcal \subset \mathbb{R}^d$ that satisfies the dense sampling assumption. Fix two vertices $x$ and $y$ of the graph and set 
\[
e_1 = (r_{low}-2\varsigma)/{r_{up}} 
\; \; , \; \;  
e_2 = r_{low}-2\varsigma.
\]
 Then the following statement holds: 
\[
e_1 D_f(x,y)
\leq 
e_2 D_{sp}(x,y) 
\leq
 D_f(x,y) - e_2.
\]
\end{proposition}

\begin{figure}[t]
\begin{center}\includegraphics[width=0.35\textwidth]{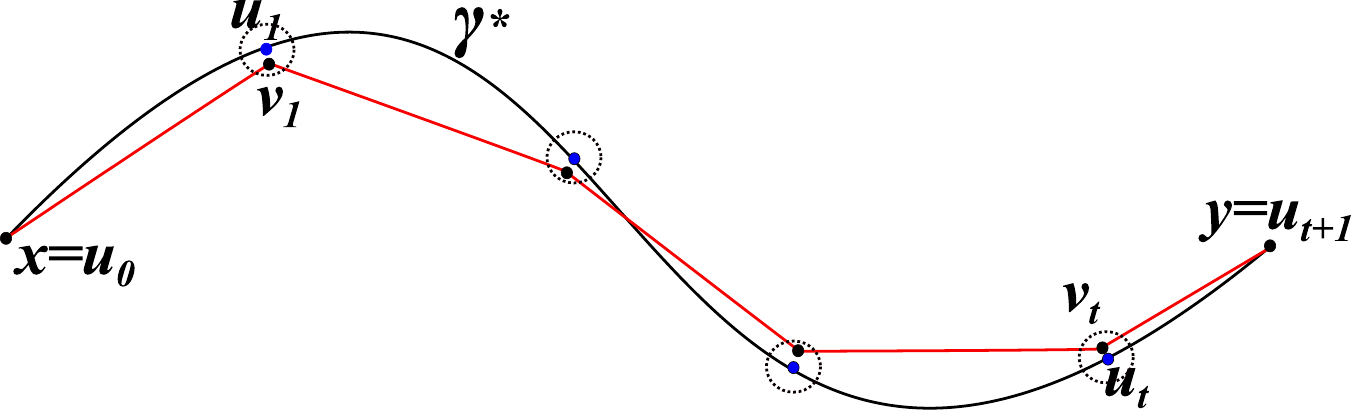}
\includegraphics[width=0.35\textwidth]{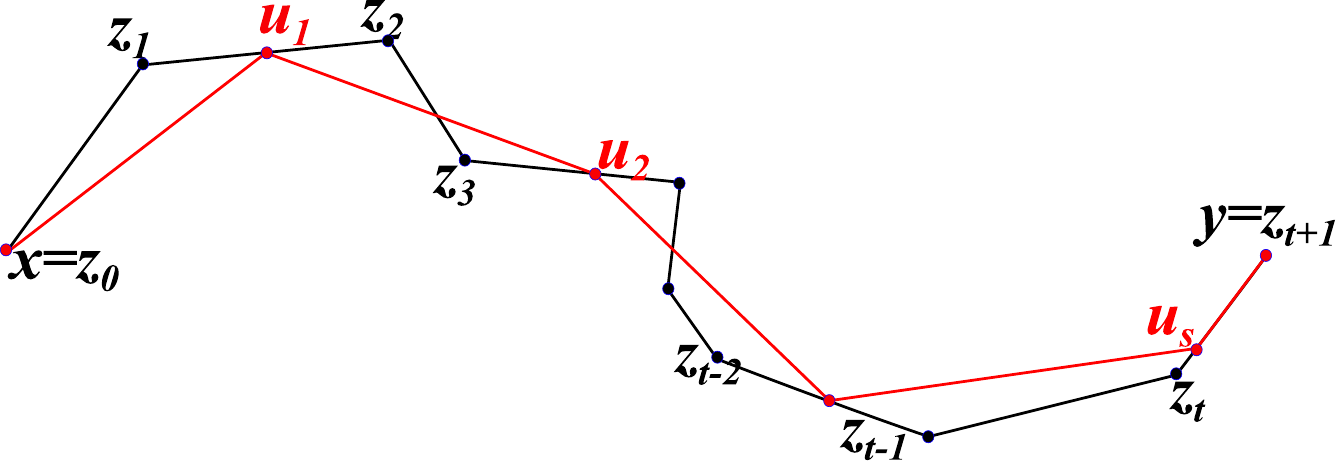}
\end{center}
\caption{Path constructions in the proofs of
  Proposition~\ref{pr:unweighted_sp} (top) and Theorem
  \ref{th:weighted} (bottom). }
\label{fig:pathconstruction}
\end{figure}

\begin{proof} \emph{Right hand side}. Consider the $f$-geodesic path $\gamma^*_{x,y}$ connecting $x$ to $y$.
Divide $\gamma^*_{x,y}$ to segments by $u_0 = x, u_1,...,u_{t},u_{t+1}=y$ such that $D_f(u_i,u_{i+1})=r_{low}-2\varsigma$ for $i=0,...,t-1$ and $D_f(u_{t},u_{t+1}) \leq r_{low}-2\varsigma$ (see Figure \ref{fig:pathconstruction}). 
Because of the dense sampling assumption, for all $i=1, ..., t$ there
exists a vertex $v_i$ in the ball $B(u_i,\varsigma; D_f)$ and we have
\begin{eqnarray*}
D_f(v_i,u_i) & \leq & \varsigma \\
D_f(u_i,u_{i+1}) & \leq &  r_{low} - 2\varsigma \\
D_f(u_{i+1},v_{i+1}) & \leq & \varsigma.
\end{eqnarray*}
Applying the triangle inequality gives $D_f(v_i,v_{i+1}) \leq r_{low}$, which shows that $v_i$ and $v_{i+1}$ are connected.
By summing up along the path we get
\begin{eqnarray*}
(r_{low}-2\varsigma) (D_{sp}(x,y)-1) & \leq & (r_{low}-2\varsigma) t \\
 =  \sum_{i=0}^{t-1}{D_f(u_i,u_{i+1})}  & \overset{(a)}{\leq} & D_f(x,y).
\end{eqnarray*}
In step (a) we use the simple fact that if $u$ is on the $f$-geodesic path from $x$ to $y$, then
\[
D_f(x,y) = D_f(x,u) + D_f(u,y).
\]
\emph{Left hand side.} Assume that the graph $\SP$ between $x$ and $y$
consists of vertices $z_0=x,z_1,...,z_s=y$. By $D_f(z_i,z_{i+1}) \leq r_{up}$ we can write
\begin{eqnarray*}
(r_{low}-2\varsigma) D_{sp}(x,y) 
 & \geq & \frac{r_{low}-2\varsigma}{r_{up}}  \textstyle{\sum_{i=0}^{s-1}{D_f(z_i,z_{i+1})}} \\
 & \geq & \frac{r_{low}-2\varsigma}{r_{up}} D_f(x,y).
\end{eqnarray*}
\end{proof} 

The next lemma uses the Lipschitz continuity and boundedness of $p$ to
show that $q(x)\|x-y\|$ is a good  approximation of $D_q(x,y)$ in
small intervals. 

\begin{lemma}[Approximating $D_q$ in small balls]
\label{lem:approx}
Consider any given $\lambda <1$. If $\|x-y\| \leq p_{\min}\lambda /L$ then
the following statements hold: 

1. We can approximate $p(y)$ by the density at $x$: 
\[
p(y)(1-\lambda) \leq p(x) \leq p(y)(1+\lambda).
\]
2. We can approximate $D_q(x,y)$ by $q(x)\|x-y\|$: 
\[
 (1-\lambda)^{1/d}  q(x)\|x-y\|
 \leq
 D_q(x,y)
 \leq
 (1+\lambda)^{1/d} q(x)\|x-y\|.
\]
\end{lemma}
\begin{proof} \textit{Part (1).}
By the Lipschitz continuity of $p$ for $\|x-y\| \leq \delta$ we have
\[
|p(x)-p(y)| \leq L \|x-y\| \leq L \delta.
\]
Setting $\delta = \lambda p_{\min} /L$ leads to  the result. 

\textit{Part (2).}
The previous part can be written as
\[
 (1-\lambda)^{1/d} q(x) 
 \leq
 q(y) 
 \leq
 (1+\lambda)^{1/d} q(x).
\]
Denote the $q$-geodesic path between $x$ and $y$ by $\gamma^*$ and the
line segment connecting $x$ to $y$ by $l$. Using the definition of a $q$-geodesic path, we can write
\begin{eqnarray*}
\int_{\gamma^*}{q(\gamma^*(t))}|\gamma^*(t)'|dt
 & \leq &
\int_{l}{q(l(t))}|l(t)'|dt  \leq \\
 (1+\lambda)^{1/d} \int_{l}{q(x)}|l(t)'|dt
  & = &
(1+\lambda)^{1/d}q(x)\|x-y\|.
\end{eqnarray*}
Also, 
\begin{eqnarray*}
\int_{\gamma^*}{q(\gamma^*(t))}|\gamma^*(t)'|dt
 & \geq &
(1-\lambda)^{1/d} \int_{\gamma^*}{q(x)}|\gamma^*(t)'|dt \\
 & \geq &
(1-\lambda)^{1/d}q(x)\|x-y\|.
\end{eqnarray*}
\end{proof}

Now we are going to show how the quantities $r_{low}$ and $r_{up}$ introduced
in Definition \ref{def:connectivity} can be bounded in random unweighted $\knn$ graphs
and how they are related to the metric $D_q$. 

To this end, we define
the  $\knn$ $q$-radii at vertex $x$ as $R_{q,k}(x) = D_q(x,y)$
and the approximated $\knn$ $q$-radii at vertex $x$ as 
$\hat{R}_{q,k}(x) = q(x) \|x - y \|$,
where $y$ is the $k$-nearest neighbor of $x$. 
The minimum and maximum values of  $\knn$ $q$-radii are defined as 
\[
R_{q,k}^{min}=\min_{u}{R_{q,k}(u)} \; \; , \; \; R_{q,k}^{max}=\max_{u}{R_{q,k}(u)}.
\]
Accordingly we define $\hat{R}_{q,k}^{min}$ and $\hat{R}_{p,k}^{max}$ for the approximated $q$-radii. 

The following proposition is a direct adaptation of Proposition 31 from \citet{LuxRadHei10_arxiv}. 
\begin{proposition}[Bounding $R_{p,k}^{min}$ and $R_{p,k}^{max}$]
\label{pr:r_low}
Given $\lambda <1/2$ define $r_{low}$ and $r_{up}$ as 
\begin{equation*}
r_{low}:=
\Big(\frac{k}{(1+\lambda) n \eta_d}\Big)^{1/d}  \;\; , \;\;
r_{up}:=
\Big(\frac{k}{(1-\lambda) n \eta_d}\Big)^{1/d}
\end{equation*}

and radius $\hat{r}_{low}$ and $\hat{r}_{up}$ as
\[
\hat{r}_{low} = \frac{r_{low}}{(1+\lambda)^{1/d}}  \;\; , \;\;
\hat{r}_{up} =  \frac{r_{up}}{(1-\lambda)^{1/d} } .
\]
Assume that $\hat{r}_{up} \leq \lambda p_{\min}^{1+1/d}/L $. Then
\begin{eqnarray*}
P\Big(R_{p,k}^{min} \leq r_{low} \Big) \leq n\exp(-\lambda^2 k/6) \\
P\Big(R_{p,k}^{max} \geq r_{up} \Big) \leq n\exp(-\lambda^2 k/6).
\end{eqnarray*}
\end{proposition} 

\begin{proof} 
Consider a ball $B_x$ with radius $\hat{r}_{low}/q(x)$ around $x$. Note that $\hat{r}_{low}/q(x) \leq p_{\min} \lambda /L$ , so we can bound the density of points in $B_x$ by $(1+\lambda) p(x)$ using Lemma \ref{lem:approx}. Denote the probability mass of the ball by $\mu(x)$, which is bounded by
\begin{eqnarray*}
\mu(x) = \int_{B_x}{p(s) ds}
& \leq &
(1+\lambda) p(x)  \int_{B_x}{ds} \\
& = & 
(1+\lambda) \hat{r}_{low}^d \eta_d =: \mu_{\max}.
\end{eqnarray*}

Observe that $\hat{R}_{q,k}(x)\leq \hat{r}_{low}$ if and only if there are at least $k$ data points in $B_x$. Let $Q \sim Binomial(n,\mu(x))$ and $S \sim Binomial(n,\mu_{\max})$. By the choice of $\hat{r}_{low}$ we have $E(S)=k/(1+\lambda)$. It follows that
\begin{eqnarray*}
P\Big(\hat{R}_{q,k}(x) \leq \hat{r}_{low} \Big) & = & P\Big(Q \geq k \Big) \leq P\Big(S \geq k \Big) \\
& = & P\Big(S \geq (1+\lambda)E(S) \Big).
\end{eqnarray*}
Now we apply a concentration inequality for binomial random
variables (see Prop. 28 in \citealp{LuxRadHei10_arxiv}) and a union bound to get
\begin{eqnarray*}
P\Big(\hat{R}_{q,k}^{min} \leq \hat{r}_{low} \Big) 
& \leq & 
P\Big(\exists i \; : \; \hat{R}_{q,k}(X_i) \leq \hat{r}_{low} \Big) \\
& \leq &
n\exp \Big(\frac{-\lambda^2 k}{3(1+\lambda)}\Big) \\
& \leq &
n\exp ({-\lambda^2 k}/6). 
\end{eqnarray*}
By a similar argument we can prove the analogous statement for $R_{p,k}^{max}$.
Finally,  Lemma \ref{lem:approx} gives 
\begin{eqnarray*}
\hat{R}_{q,k}^{min} \geq \frac{R_{q,k}^{min}}{(1+\lambda)^{1/d}}  \;,\; 
\hat{R}_{q,k}^{max} \leq \frac{R_{q,k}^{max}}{(1-\lambda)^{1/d}} .
\end{eqnarray*}
\end{proof}

The following proposition shows how the sampling parameter $\varsigma$ can be
chosen to satisfy the dense sampling assumption. Note that we decided to choose $\varsigma$ in a form 
that keeps the statements in Theorem \ref{th:main} simple, rather than optimizing
over the parameter $\varsigma$ to maximize the probability of success. 


\begin{lemma}[Sampling lemma]
\label{lem:sampling}
Assume $X_1,...,X_n$ are sampled i.i.d. from a probability
distribution $p$ and a constant $a<1$ is given. Set $\lambda$ as in Theorem \ref{th:main}, 
\[
\varsigma := (1+\lambda)^{1/d} \Big( \frac{k^a}{\eta_d n }
\Big)^{1/d}, 
\]
and $e_3(\lambda) := 2^d/(1-\lambda)^2$. Then with probability at least $1- e_3 n \exp(-k^a /6)$, for every $x \in \Xcal$ exists a $y \in X_1,...,X_n$ such that $D_q(x,y)\leq \varsigma$.
\end{lemma}


\begin{proof}
Define $\varsigma_0 =  (1+\lambda)^{-1/d} \varsigma $. We prove that  for every $x \in \Xcal$, there exist a vertex $y$ such that $q(x)\|x-y\| \leq \varsigma_0$. Then using Lemma \ref{lem:approx} will give the result.
 
The proof idea is a generalization of the covering argument in the
proof of the {sampling lemma} in \citet{TenSilLan00withproofs}. 
We first construct a covering of $\Xcal$ that consists of balls with approximately the same probability
mass. The centers of the balls are chosen by an iterative procedure that
ensures that no center is contained in any of the balls we have so
far. We choose the radius $\varsigma_0 / q(x)$ for the ball at point $x$ and call it $B_q(x,\varsigma_0)$. The probability mass of this ball can be bounded by
\[
\mathcal{V}(B_q(x,\varsigma_0)) \geq (1-\lambda) \varsigma_0^d \eta_d 
\]
Note that smaller balls $B_q(u,(1-\lambda)^{1/d}\varsigma_0/2)$ are
all disjoint. To see this, consider two balls
$B_q(x,(1-\lambda)^{1/d}\varsigma_0/2)$,
$B_q(y,(1-\lambda)^{1/d}\varsigma_0/2)$. Observe that 
\[
\frac{(1-\lambda)^{1/d}\varsigma_0}{2q(x)}+\frac{(1-\lambda)^{1/d}\varsigma_0}{2q(y)}
\leq
\frac{\varsigma_0}{q(x)}.
\]
We can bound the total number of balls by
\[
S \leq \frac{1}{\mathcal{V}(B_q(x,(1-\lambda)^{1/d}\varsigma_0/2))} 
\leq
\frac{2^d}{\eta_d(1-\lambda)^2 \varsigma_0^{d}}.
\]
Now we sample points from the underlying space and apply the same concentration inequality as above. We bound the probability that a ball $B_q(u,\varsigma_0)$ does not contain any
sample point (``is empty'') by 
\[
Pr(\text{Ball }i\text{ is empty}) 
\leq
\exp(-n \varsigma_0^d \eta_d /6).
\]
Rewriting and Substituting the value of $\varsigma_0$ gives 
\ba
& Pr(\text{no ball is empty}) 
\;\; \geq \;\; 
1-\textstyle{\sum_i}{Pr(B_i\text{ is empty})} \\
& \geq 
1-S \cdot e^{-n \varsigma_0^d \eta_d /6} 
\;\; \geq \;\; 
1- \frac{2^d n e^{-k^a /6}}{(1-\lambda)^2 k^a} \\
&  \geq 
1- \frac{2^d n e^{-k^a /6}}{(1-\lambda)^2} 
\;\; = \;\;
1- e_3 n e^{-k^a /6}.
\ea
\vspace{-3pt}
\end{proof}
%

\textbf{Proof of Theorem \ref{th:main}.}
Set $r_{low}$ and $r_{up}$ as in Proposition \ref{pr:r_low}. The assumption on $\lambda $ ensures that $\hat{r}_{up} \leq \lambda p_{\min}^{1+1/d}/L $. It follows from Proposition
\ref{pr:r_low} that the statements about $r_{low}$ and $r_{up}$ in
Definition \ref{def:connectivity} both hold for $G_n$  with probability at least $\mu_1 = 1- 2n\exp(-\lambda^2 k/6)$.  
Set $\varsigma$ as in Lemma \ref{lem:sampling} and define the constant 
$a < 1 - \log_{k}\Big( 4^d(1+\lambda)^2 \Big).$
By this choice we have $r_{low} > 4\varsigma $. Lemma \ref{lem:sampling} shows that the
sampling assumption holds in $G_n$ for the selected $\varsigma$  with probability at least
$\mu_2 = 1- e_3 n \exp(-k^a /6)$.
Together, all these statements about $G_n$ hold with probability at least $\mu := 1 - 3 e_3 n\exp(-\lambda^2 k^a/6)$. 

Using Proposition \ref{pr:unweighted_sp} completes the first part of the theorem. 
For the convergence we have
\[
e_1 = \frac{r_{low} - 2\varsigma}{r_{up}} = \Big(\frac{1-\lambda}{1+\lambda} \Big)^{1/d} - 2 \Big(\frac{1-\lambda^2}{k^{1-a}} \Big)^{1/d}.
\]
This shows that $e_1 \rightarrow 1$ as $\lambda \rightarrow 0$ and $k \rightarrow \infty$. For $\lambda \rightarrow 0$ and $k \rightarrow \infty$ we can set $a$ to any constant smaller than $1$. Finally it is easy to check that $e_2 \rightarrow 0$ and $\varsigma / r_{low} \rightarrow 0$.
\ulesqed

\section{Shortest paths in weighted graphs}
\label{sec:sp_weighted}

In this section we discuss both questions from the Introduction. We
also extend our results from the previous section to weighted $\knn$ graphs and $\eps$-graphs.

\subsection{Weight assignment problem}
\label{subsec:weight_ass_problem}

Consider a graph based on the i.i.d.\ sample $X_1, ..., X_n \in \Xcal$
from the density $p$. We are given a positive scalar function $f$ which is only
a function of the density: $f(x) = \tilde{f}(p(x))$.
We want to assign edge weights such that the graph $\SP$ distance converges to the $f$-distance in $\Xcal$. 

It is well known that the $f$-length of a curve $\gamma : [a,b]
\rightarrow \Xcal$ can be approximated by a Riemann sum over a partition of $[a,b]$ to subintervals $[x_i,x_{i+1}]$:  
\[
\textstyle{\hat{D}_{f,\gamma} = \sum_i{f\Big(\frac{\gamma(x_{i})+\gamma(x_{i+1})}{2} \Big) \|\gamma(x_{i})-\gamma(x_{i+1})\|}}.
\]
As the partition gets finer, the approximation $\hat{D}_{f,\gamma}$ converges to $D_{f,\gamma}$ (cf. Chapter 3 of \citealp{Gamelin2001}). This suggests using edge weights 
\[
w_{ij} = \tilde{f}\textstyle{\Big(p(\frac{X_i+X_j}{2})\Big)}\|X_i - X_j\|. 
\]
However the underlying density $p(x)$ is not known in many machine learning
applications. \citet{Sajama2005} already proved that the plug-in approach using
a  kernel density estimator $\hat p(x)$ for $p(x)$ will lead to
the convergence of the $\SP$ distance to $f$-distance in
$\eps$-graphs. Our next result hows how to choose edge weights in
$\knn$ graphs without estimating the density. It is a corollary from a
theorem that will be presented in Section \ref{ssec-limit}. 

We use a notational convention to simplify our arguments and hide
approximation factors that will eventually go to $1$ as the sample
size goes to infinity. We say that $f$ is approximately larger than
$g$ ($f \succcurlyeq_{\lambda}  g$) if there exists a function $e(\lambda)$ such that $f \geq e(\lambda) g$ and $e(\lambda) \rightarrow 1$ as $n \rightarrow \infty$ and $\lambda \rightarrow 0$. The symbol $\preccurlyeq_{\lambda}$ is defined similarly. We use the notation $f \approx_{\lambda} g$ if $f \preccurlyeq_{\lambda} g$ and $f \succcurlyeq_{\lambda}  g$.
\begin{corollary}
\label{th:edge_weights}
{\bf (Weight assignment)} 
Consider the $\knn$ graph based on the i.i.d. sample $X_1,
..., X_n \in \Xcal$ from the density $p$. Let $f$ be of the form $f(x) = \tilde{f}(p(x))$ with $\tilde{f}$ increasing. We assume that $\tilde{f}$ is Lipschitz continuous and $f$ is bounded away from 0. Define $r = (k/(n \eta_d))^{1/d}$ and set the edge weights
\vspace{-4pt}
\begin{equation}
\label{eq:weight_w}
w_{ij} = \|X_i - X_j\| \tilde{f}\Big(\frac{r^d}{\|X_i - X_j\|^d}\Big).
\end{equation}
Fix two points $x=X_i$ and $y = X_j$. Choose $\lambda$ and $a$ as in
Theorem \ref{th:weighted}. 
Then with probability at least 
$1 - 3 e_3 n\exp(-\lambda^2 k^a/6)$
we have $D_{sp}(x,y)\approx_{\lambda} D_f(x,y)$.
\end{corollary}
\vspace{-8pt}

%
%


\subsection{Limit distance problem} \label{ssec-limit}
Consider a weighted graph based on the i.i.d. sample $X_1, ..., X_n
\in \Xcal$ from the density $p$. We are given a increasing
edge weight function $h:\mathbb{R}^{+}\rightarrow \mathbb{R}^{+}$ which
assigns weight $h(\|x-y\|)$ to the edge $(x,y)$. We are interested in
finding the limit of the graph $\SP$ distance with respect to edge
weight function $h$ as the sample size goes to infinity. In particular
we are looking for a distance function $f$ such that the $\SP$
distance converges to the $f$-distance.

Assume we knew the solution $f^*=\tilde{f^*}(p(x))$ of this
problem. To guarantee the convergence of the distances, $f^*$ should
assign weights of the form of $w_{ij} \approx
\tilde{f^*}(p(X_i))\|X_i-X_j\|$. This would mean 
\[
 \tilde{f^*}(p(X_i))  \approx \frac{h(\|X_i-X_j\|)}{\|X_i-X_j\|}, 
\]
which shows that determining $\tilde{f^*}$ is closely related to finding a density based estimation for $\|X_i-X_j\|$.

Depending on $h$, we distinguish two regimes for this problem:
subadditive and superadditive.

\subsubsection{Subadditive weights}
A function $h(x)$ is called subadditive if $\forall x,y\geq 0 :
h(x)+h(y)\geq h(x+y)$. Common examples of subadditive functions are
$f(x)=x^a$, $a<1$ and $f(x)=xe^{-x}$. For a subadditive $h$, the $\SP$ in the graph will satisfy the
triangle inequality and it will prefer jumping along distant
vertices. Based on this intuition, we come up with the following guess for vertices along the $\SP$: For $\epsilon$-graphs we have the approximation $\|X_i-X_j\| \approx \epsilon$ and $f(x)=h(\epsilon)/\epsilon$. For $\knn$-graphs we have $\|X_i-X_j\| \approx r/q(X_i)$ with $r = (k/(n \eta_d))^{1/d}$ and 
\vspace{-6pt}
\begin{equation*}
\label{eq:f_h}
f(x)=h(\frac{r}{q(x)}) \frac{q(x)}{r} \; , \; \tilde{f}(x)=h(\frac{r}{x^{1/d}}) \frac{x^{1/d}}{r}.
\vspace{-4pt}
\end{equation*}
We formally prove this statement for $\knn$ graphs in the next theorem. In contrast to Theorem \ref{th:main}, the scaling factor is moved into $f$. The
proof for $\epsilon$-graphs is much simpler and can be adapted by
setting $r = \epsilon$, $q(x)=1$, and $r_{low}=r_{up} = \epsilon$.

\vspace{-5pt}
\begin{theorem}[Limit of $\SP$ in weighted graphs]
\label{th:weighted}
Consider the $\knn$ graph based on the i.i.d. sample $X_1,
..., X_n \in \Xcal$ from the density $p$. Let $h$ be an increasing, Lipschitz continuous and subadditive function, and define the edge weights $w_{ij} = h(\|X_i-X_j\|)$. Fix two points $x=X_i$ and $y=X_j$. Define $r = (k/(n \eta_d))^{1/d}$ and set
\vspace{-5pt}
\[
f(x)=h(\frac{r}{q(x)}) \frac{q(x)}{r}.
\]
Choose $\lambda$ and $a$ such that 
\[
\lambda \geq \frac{4 L}{\eta_d^{1/d} p_{\min}^{1+1/d} } \Big(\frac{k}{n}\Big)^{1/d}
\; , \;
a < 1 - \log_{k}\Big( 4^d(1+\lambda)^2 \Big).
\]
Then with probability at least 
$1 - 3 e_3 n\exp(-\lambda^2 k^a/6)$
we have $D_{sp}(x,y)\approx_{\lambda} D_f(x,y)$.
\end{theorem}

\begin{proof}
The essence of the proof is similar to the one in Theorem
\ref{th:main}, we present a sketch only.  
The main step is to adapt Proposition \ref{pr:unweighted_sp} to
weighted graphs with weight function $h$. Adapting Lemma
\ref{lem:approx} for general $f$ is straightforward. The lemma states
that $D_f(x,y)\approx_{\lambda} f(x)\|x-y\|$ for nearby points. We set
$r_{low}$ and $\varsigma$ as in the sampling lemma and Proposition
\ref{pr:r_low} (these are properties of $\knn$ graphs and hold for any $f$). Proposition \ref{pr:r_low} says that in $\knn$ graphs, $x$ is connected to $y$ with high probability iff $\|x-y\| \preccurlyeq_{\lambda} r/q(x)$. The probabilistic argument and the criteria on choosing $\lambda$ are similar to Theorem \ref{th:main}.  

First we show that $D_{sp}(x,y) \preccurlyeq_{\lambda} D_f(x,y)$. 
Consider the $f$-geodesic path $\gamma^*_{x,y}$ connecting $x$ to $y$.
Divide $\gamma^*_{x,y}$ into segments $u_0 = x, u_1,...,u_{t},u_{t+1}=y$ such that $D_q(u_i,u_{i+1}) = r_{low}-2\varsigma$ for $i=0,...,t-1$ and $D_q(u_{t},u_{t+1}) \leq r_{low}-2\varsigma$ (see Figure \ref{fig:pathconstruction}). There exists a vertex $v_i$ near to $u_i$ such that $v_i$ and $v_{i+1}$ are connected. We show that the length of the path $x, v_1, ..., v_t, y$ is approximately smaller than $D_f(x,y)$. 
From the path construction we have 
\[
\|v_i-v_{i+1}\| \approx_{\lambda} \|u_i-u_{i+1}\| \approx_{\lambda} r/q(u_i).
\] 
By summing up along the path we get
\begin{eqnarray*}
D_{sp}(x,y) 
 & \leq & \textstyle{\sum_i{h(\|v_i-v_{i+1}\|)}} \\
 & \approx_{\lambda} & \textstyle{\sum_i{h(\|u_i-u_{i+1}\|)}}  \approx_{\lambda}  \textstyle{\sum_i{h(\frac{r}{q(u_i)})}} \\
 & = & \textstyle{\sum_i{f(u_i)\frac{r}{q(u_i)}}} \approx_{\lambda}  \textstyle{\sum_i{f(u_i)\|u_i-u_{i+1}\|}} 
\end{eqnarray*}
From the adaptation of Lemma \ref{lem:approx} we have
$D_f(u_i,u_{i+1})\approx_{\lambda} f(u_i)\|u_i-u_{i+1}\|$, which gives
\[
\textstyle{\sum_i{f(u_i)\|u_i-u_{i+1}\|}} \approx_{\lambda} \textstyle{\sum_i{D_f(u_i,u_{i+1})}} = D_f(x,y).
\]
This shows that $D_{sp}(x,y) \preccurlyeq_{\lambda} D_f(x,y)$. 

For the
other way round, we use a technique different from Proposition \ref{pr:unweighted_sp}. 
Denote the graph shortest path between $x$ and $y$ by $\pi:
z_0=x,z_1,...,z_s,z_{s+1}=y$. Consider $\pi'$ as a continuous path in
$\Xcal$ corresponding to $\pi$. As in the previous part, divide $\pi'$ into segments $u_0 = x, u_1,...,u_{t},u_{t+1}=y$ (see Figure \ref{fig:pathconstruction}). 
From $D_q(z_i,z_{i+1}) \preccurlyeq_{\lambda} r$ and $D_q(u_i,u_{i+1})
\approx_{\lambda} r$ we have $s \succcurlyeq_{\lambda} t$. Using this
and the subadditivity of $h$ we get
\begin{equation*}
\textstyle{D_{sp}(x,y) = \sum_i{h(\|z_i-z_{i+1}\|)} \succcurlyeq_{\lambda}  \sum_i{h(\|u_i-u_{i+1}\|)}.}
\end{equation*}
To prove  $D_{sp}(x,y) \succcurlyeq_{\lambda}
D_f(x,y)$, we can write
\begin{eqnarray*}
\textstyle{\sum_i{h(\|u_i-u_{i+1}\|)}} & \approx_{\lambda} & \textstyle{\sum_i{h(\frac{r}{q(u_i)})}} =  \textstyle{\sum_i{\frac{f(u_i)r}{q(u_i)}}} \\
& \approx_{\lambda} & \textstyle{\sum_i{f(u_i)\|u_i-u_{i+1}\|}} \\
& \approx_{\lambda} & \textstyle{\sum_i{D_f(u_i,u_{i+1})}} \geq  D_f(x,y).
\vspace{-10pt}
\end{eqnarray*}
\end{proof}

\vspace{-5pt}
The \textbf{proof of Theorem \ref{th:edge_weights}} is a direct
consequence of this theorem. It follows by choosing 
$ h(t) = t \tilde{f}( {r^d}/ {t^d})$ (which is subadditive if $\tilde
f$ is increasing) and setting $w_{ij} =
h(\|X_i-X_j\|)$. 

\vspace{-5pt}
\subsubsection{Supperadditive weights} 
A function $h$ is called superadditive if $\forall x,y\geq 0 :
h(x)+h(y)\leq h(x+y)$. Examples are $f(x)=x^a;a>1$ and $f(x)=xe^{x}$.
To get an intuition on the behavior of the $\SP$ for a superadditive $h$, take an
example of three vertices $x,y,z$ which are all connected in the graph
 and sit on a straight line such that $\|x-y\| + \|y-z\| = \|x-z\|.$ 
By the superadditivity, the $\SP$ between $x$ and $z$ will prefer going through $y$ rather
than directly jumping to $z$. More generally, the graph $\SP$ will prefer
taking many ``small'' edges rather than fewer ``long'' edges. For this reason, we do not
expect a big difference between superadditive weighted $\knn$ graphs
and $\epsilon$-graphs: the long edges in the $\knn$ graph
will not be used anyway. However, due to technical problems we did not
manage to prove a formal theorem to this effect. 

The special case of the superadditive family $h(x)=x^a$, $a>1$ is treated in \citet{Hwang2012} by completely different methods. Although their results are presented for complete graphs, we believe that it can be extended to $\epsilon$ and $\knn$ graphs.
We are not aware of any other result for the limit of $\SP$ distance
in the superadditive regime.



\section{Consequences in applications}
\label{sec:conseq} 
In this section we study the consequences of our results on manifold embedding using Isomap and on a particular semi-supervised learning method. 

There are two cases where we do not expect a drastic difference
between the $\SP$ in weighted and unweighted $\knn$ graph: (1) If the
underlying density $p$ is close to uniform. (2) If
the intrinsic dimensionality of our data $d$ is high. The latter is
because in the $q$-distance, the underlying density arises in
the form of $p(x)^{1/d}$, where the exponent flattens the distribution for large
$d$.
\vspace{-5pt}
\subsection{Isomap}
Isomap is a widely used method for low dimensional manifold embedding
\citep{TenSilLan00withproofs}. The main idea is to use metric
multidimensional scaling on the matrix of pairwise geodesic
distances. Using the Euclidean length of edges as their weights will
lead to the convergence of the $\SP$ distance to the geodesic distance. But what would be the effect of applying Isomap to unweighted graphs? 

Our results of the last section already hint that there is no big difference between unweighted and weighted $\epsilon$-graphs for Isomap. 
However, the case of $\knn$ graphs is different because weighted and unweighted shortest paths measure different quantities. 
The effect of applying Isomap to unweighted $\knn$ graphs can easily
be demonstrated by the following simulation. We sample 2000 points
in $\R^2$ from a distribution that has two uniform high-density
squares, surrounded by a uniform low density region. An unweighted $\knn$ graph is constructed with
$k=10$, and we apply Isomap with target dimension 2. The result is
depicted in Figure \ref{fig:isomap_knn}. We can see that the Isomap
embedding heavily distorts the original data: it stretches high
density regions and compacts low density regions to make the vertex
distribution close to uniform. 

\subsection{Semi-supervised learning}

Our work has close relationship to some of the literature on
semi-supervised learning (SSL). In regularization based approaches,
the underlying density is either exploited implicitly as attempted in
Laplacian regularization (\citealp{ZhuGhaLaf03} but see 
\citealp{NadSreZho09,AlaLux11} and \citealp{ZhouB11}),  or more explicitly as in measure based regularization \citep{BouChaHei03}.
Alternatively, one defines new distance functions on the data that take the density of the unlabeled points into account.
Here, the papers by \citet{Sajama2005} and \citet{Bijral11} are most related to our paper. Both papers suggest different ways to approximate the density based distance from the data. In \citet{Sajama2005} it is achieved by estimating the underlying density while in \citet{Bijral11}, the authors omit the density estimation and use an approximation. 

Our work shows a simpler way to converge to a similar distance
function for a specific family of $f$-distances, namely constructing a $\knn$ graph and assigning edge
weights as in Equation \ref{eq:weight_w}.
%

\section{Conclusions and outlook} 

We have seen in this paper that the shortest path distance on
unweighted $\knn$ graphs has a very funny limit behavior: it prefers
to go through regions of low density and even takes large detours in
order to avoid the high density regions. In hindsight, this result seems
obvious, but most people are surprised when they first
hear about it. In particular, we believe that it is important to
spread this insight among machine learning practitioners, who
routinely use unweighted $\knn$-graphs as a simple, robust alternative
to $\eps$-graphs.

In some sense, unweighted $\eps$-graphs and unweighted $\knn$ graphs
behave as ``duals'' of each other: while
degrees in $\eps$-graphs reflect the underlying density, they are independent of the density in $\knn$
graphs. While the shortest path in
$\eps$-graphs is independent of the underlying density and converges
to the Euclidean distance, the shortest paths in $\knn$
graphs take the density into account. 

Current practice
is to use $\eps$ and $\knn$ graphs more or less interchangeably in
many applications, and the decision for one or the other graph is
largely driven by robustness or convenience considerations. However,
as our results show it is important to be aware of the implicit 
consequences of this choice.  Each graph carries different information about the underlying
density, and depending on how a particular machine learning algorithms makes use
of the graph structures, it might either miss out or benefit from this
information.



\bibliography{knn_icml2012,ules_publications,general_bib}
\bibliographystyle{icml2012}

\end{document}

%% file: ules_latex_defaults.tex
\usepackage{amssymb}
\usepackage{amsmath}
\allowdisplaybreaks[3] 

\usepackage{fancyhdr} 
\usepackage{ifthen} 
\usepackage{layout} 
\usepackage{xspace}
\usepackage{url}
\usepackage{textcomp} 

\usepackage{wasysym} 

\usepackage[latin1]{inputenc} 

\def\ba#1\ea{\begin{align*}#1\end{align*}} 
\def\banum#1\eanum{\begin{align}#1\end{align}} 

\newcommand{\bi}{\begin{itemize}}
\newcommand{\ei}{\end{itemize}}
\newcommand{\be}{\begin{enumerate}}
\newcommand{\ee}{\end{enumerate}}
\newcommand{\bc}{\begin{center}}
\newcommand{\ec}{\end{center}}

\def\bv#1\ev{\begin{verbatim}#1\end{verbatim}} 

\ifx\lemma\undefined
\newtheorem{theorem}{Theorem} 
\newtheorem{lemma}[theorem]{Lemma} 
\newtheorem{proposition}[theorem]{Proposition} 

\newtheorem{corollary}[theorem]{Corollary}
\newtheorem{definition}[theorem]{Definition}

\else
\fi

\newcommand{\ulesqed}{\hfill\smiley}

\ifx\proof\undefined
\newenvironment{proof}{\par\noindent{\em Proof. }}{\ulesqed}
\else
\fi

\definecolor{ulesred}{rgb}{0.7,0,0}

% ...

\newcommand{\RR}{\mathbb{R}}

\let\R\undefined 
\newcommand{\R}{\RR}

\newcommand{\Xcal}{\mathcal{X}}

\newcommand{\knn}{\text{kNN}}

\newcommand{\eps}{\ensuremath{\varepsilon}}
\renewcommand{\epsilon}{\ensuremath{\varepsilon}}
\renewcommand{\phi}{\ensuremath{\varphi}}
\wasyfamily
\DeclareFontShape{U}{wasy}{b}{n}{ <-10> ssub * wasy/m/n
<10> <10.95> <12> <14.4> <17.28> <20.74> <24.88>wasyb10 }{}

\DeclareFontShape{U}{wasy}{m}{n}{ <5> <6> <7> <8> <9> gen * wasy
<10> <10.95> <12> <14.4> <17.28> <20.74> <24.88> <35> <40> <50> <60> wasy10  }{}
